\documentclass[letterpaper]{article} 
\usepackage{aaai25}  
\usepackage{times}  
\usepackage{helvet}  
\usepackage{courier}  
\usepackage[hyphens]{url}  
\usepackage{graphicx} 
\usepackage{natbib}  
\usepackage{caption} 
\usepackage{algorithm}
\usepackage{algorithmic}
\usepackage{amsmath,amssymb,amsfonts,amsthm}
\usepackage{makecell}
\usepackage{tabularray}
\usepackage{xcolor}

\usepackage{newfloat}
\usepackage{listings}
\DeclareCaptionStyle{ruled}{labelfont=normalfont,labelsep=colon,strut=off} 
\floatstyle{ruled}
\newfloat{listing}{tb}{lst}{}
\floatname{listing}{Listing}
\title{DCC: Differentiable Cardinality Constraints for Partial Index Tracking}
\author{
    Wooyeon Jo\textsuperscript{\rm 1,3}, 
    Hyunsouk Cho\textsuperscript{\rm 1,2}\thanks{Corresponding Author}
}
\affiliations{
    Department of Artificial Intelligence\textsuperscript{\rm 1}\\
    Department of Software and Computer Engineering\textsuperscript{\rm 2}\\
    Ajou University, Suwon 16499, Republic of Korea \\
    The AI Lab, Inc.\textsuperscript{\rm 3}, Seoul 06177, Republic of Korea\\

    qt5828@ajou.ac.kr, 
    hyunsouk@ajou.ac.kr
}

\usepackage{bibentry}

\begin{document}

\maketitle

\newcommand{\eg}{{\it e.g.}}%
\newcommand{\ie}{{\it i.e.}}%
\newcommand{\etal}{{\it et al.}}%
\newcommand{\etc}{{\it etc}}%
\newcommand{\com}{\textcolor{red}}
\newcommand{\ours}{\textbf{DCC}}
\newcommand{\oursfpp}{$\textbf{DCC}_{fpp}$}

\begin{abstract}
    Index tracking is a popular passive investment strategy aimed at optimizing portfolios, but fully replicating an index can lead to high transaction costs. To address this, partial replication have been proposed. However, the cardinality constraint renders the problem non-convex, non-differentiable, and often NP-hard, leading to the use of heuristic or neural network-based methods, which can be non-interpretable or have NP-hard complexity. To overcome these limitations, We propose a Differentiable Cardinality Constraint (\ours) for index tracking and introduce a floating-point precision-aware method (\oursfpp) to address implementation issues. We theoretically prove our methods calculate cardinality accurately and enforce actual cardinality with polynomial time complexity. We propose the range of the hyperparameter $a$ ensures that \oursfpp\ has no error in real implementations, based on theoretical proof and experiment. Our method applied to mathematical method outperforms baseline methods across various datasets, demonstrating the effectiveness of the identified hyperparameter $a$.
\end{abstract}

%

\section{Introduction}

\quad Index tracking, particularly through full replication, is one of the most widely used strategies in portfolio optimization. This approach constructs a portfolio that mimics a specific market index by including all constituent stocks with corresponding weights. Full replication can be effectively solved as a basic regression problem using mathematical optimization techniques, enabling the efficient and precise portfolio construction. However, this method assigns continuous weights to all stocks in the portfolio, leading to significant transaction costs—a critical challenge in real-world investment scenarios. To mitigate these costs, partial replication has been proposed \cite{partial}, \cite{survey:partial}, where only a subset of stocks is assigned weights, reducing the overall number of transactions. Partial replication extends full replication by  incorporating a cardinality constraint to limit the number of stocks.

Cardinality constraints, integral to partial replication, exhibit several notable technical challenges~\cite{cardinality, NP_hard}:
i) Discreteness: Cardinality constraints enforce a limit on the number of selected stocks, resulting in a discrete solution space, unlike the continuous one encountered in full replication.
ii) Combinatorial Complexity: These constraints give rise to a combinatorial optimization problem, where all possible combinations must be considered. Their independent and non-continuous nature complicates to reformulate the problem into a form that can be solved using traditional mathematical optimization techniques, such as those requiring linearity, convexity, or differentiability.
iii) Computational Complexity: Finding a solution that satisfies the cardinality constraint is classified as an NP-hard problem, characterized by high computational complexity, making it difficult to identify efficient solutions.
Due to these inherent characteristics, traditional mathematical optimization approaches, which were effective for solving full replication problems, struggle with partial replication. Consequently, heuristic methods \cite{heuristic-evolutionary,clustering,meta-heuristic,meta-heuristic-search,SNN} have been proposed to address partial replication. However, these heuristic approaches have significant drawbacks, including the non-interpretability of some solution processes and the persistence of high complexity.

To overcome these limitations, it would be advantageous to transform cardinality constraints into a form that can be tackled using mathematical optimization techniques. Thus, we propose the Differentiable Cardinality Constraint (\ours), which is not only adaptable to mathematical optimization techniques but also ensures the enforcement of actual cardinality constraints.
In summary, our contributions are as follows:
\begin{enumerate}
    \item We propose \ours, applicable to any optimization algorithm handling differentiable constraints, particularly using the Lagrangian multiplier method for partial replication.
    \item  To address implementation challenges, we introduce a floating-point precision-aware variant, \oursfpp, ensuring accurate enforcement of cardinality constraints.
    \item We establish conditions for the constant $a$ in \oursfpp, providing accurate cardinality calculations and constraint enforcement.
    \item We validate \oursfpp's performance in partial replication using the SLSQP method, demonstrating improved results on real market data and yielding more precise, interpretable solutions within polynomial time.
\end{enumerate}

\section{Related Works}

\quad Full replication is a passive investment strategy in portfolio optimization, where objective is to minimize the tracking error between a target index and the portfolio index, which can be formulated as a constrained regression problem. This problem can be efficiently solved using mathematical optimization techniques. Specifically, the constraints in full replication include the sum-to-one constraint, where the sum of portfolio weights equals one, and the non-negativity constraint, ensuring that each weight is non-negative. Both constraints are linear, making Quadratic Programming (QP) an effective method for efficiently solving full replication problems, as demonstrated in various studies~\cite{markowitz2,markowitz,convex}.
Furthermore, since these constraints can also be expressed in differentiable forms, full replication can be solved using Lagrangian multipliers~\cite{lagrangian,lagrange}. These mathematical optimization techniques are easily implemented using libraries such as CVXPY~\cite{cvpxy} or SciPy~\cite{scipy}, which efficiently find precise solutions. Since these methods follow well-established mathematical procedures, the resulting portfolio solutions are interpretable, as the clear objective functions and explicit constraints make it easy to understand how each decision impacts the final outcome. However, the cardinality constraint is neither linear nor differentiable, making it challenging to solve using mathematical optimization methods. 

To address this problem, heuristic approaches have been employed to address partial replication problems. Heuristic methods such as search algorithms \cite{meta-heuristic-search}, which iteratively explore different combinations of stocks to identify those that optimize the portfolio, and meta-heuristic approaches, including evolutionary algorithms~\cite{heuristic-evolutionary,meta-heuristic}, have been used. Additionally, clustering methods~\cite{clustering} have also been employed to select optimal subsets of stocks, effectively reducing the portfolio size while attempting to maintain tracking accuracy. 
Despite the practical utility of these heuristic methods, they come with inherent limitations. The approximate nature of heuristic solutions means they may find suboptimal solution, and the large search space involved in these problems introduces significant computational complexity. Thus, recently, \cite{SNN} have proposed the use of neural network-based approaches for partial replication, employing reparameterization techniques to transform unconstrained parameters into forms that satisfy the cardinality constraint. However, these neural network approaches often function as black-box models, obscuring the interpretability of the solutions and the intermediate steps involved.

The limitations of heuristic and neural network approaches underscore the need for a mathematical optimization approach to solve the partial replication problem efficiently. Traditional methods of handling the cardinality constraint, such as iteratively applying full replication and selecting the top $K$ stocks or excluding $N-K$ stocks with the lowest weights, can satisfy the constraint but at the cost of increased complexity and inefficiency. This approach is straightforward but often yields suboptimal results due to its iterative nature and lack of optimization techniques.
Unlike these previous methods, we tackle the partial replication problem by proposing the Differentiable Cardinality Constraint (\ours), which transforms the cardinality constraint into a differentiable form. This innovation allows it to be directly integrated into mathematical optimization frameworks, enabling efficient and precise resolution of the problem in polynomial time, while still producing interpretable solutions.

\section{Preliminaries}
\quad Before introducing our Differentiable Cardinality Constraint (\ours), it is essential to formally define the index tracking and the cardinality constraint associated with it.

\subsection{Full Replication}
Traditional index tracking (full replication) involves constructing a portfolio to minimize the difference between the market index and the portfolio index, i.e. tracking error. 
Minimizing the tracking error is a straightforward regression problem when dealing with $N$ stocks over a duration $D$.
The objective function is
$\min\|X\mathbf{w} - \mathbf{y}\|^2_2$
where $X \in \mathbb{R}^{D \times N}$ is the daily return of stocks and $\textbf{w} \in [0,1]^N$ such that $\mathbf{w} = [w_1 w_2 \ldots w_N]^T$ is a weight vector of portfolio. $w_i$ is the weight of $i$-th stock and $\textbf{y} \in \mathbb{R}^D $ is the daily target index. 

Moreover, the portfolio must satisfy straightforward constraints: each stock should have a non-negative weight, and the sum of all weights must equal one. Then, we can define the full replication problem as:
\begin{equation}
\begin{aligned}
     &\min_{\textbf{w}} \quad \|X\mathbf{w} - \mathbf{y}\|^2_2 \\
    \text{subject to:}&\quad  w_i \geq 0 \ \ \forall i ,
     \quad \sum_{i=1}^{N}w_i = 1 
\end{aligned}
\end{equation}
However, full replication assigns continuous weights to all stocks, which leads to high transaction costs. Therefore, a cardinality constraint, which limits the number of stocks in the portfolio, is necessary to mitigate these costs.

\subsection{Partial Replication}
The partial replication ensures that the portfolio's cardinality, calculated through a specific function, does not exceed a given value $K$. To enforce this constraint, we first define the function that calculates the portfolio's cardinality. Let $w_i$ represent the weight of the $i$-th stock in the portfolio. Then cardinality function is defined using a binary function, $b(w_i)$, that assigns a value of 0 if a portfolio weight is zero, and 1 if the weight is greater than zero. By summing the binary function values across all weights in the portfolio, we can calculate the portfolio's cardinality. 
 Partial replication extends the full replication approach by incorporating a cardinality constraint. Using the cardinality function, partial replication can be expressed as follows:
 \begin{equation}
\begin{aligned}
    &\min_{\textbf{w}} \quad \|X\mathbf{w} - \mathbf{y}\|^2_2 \\
    \text{subject to:}&\quad  w_i \geq 0 \ \ \forall i ,
     \quad \sum_{i=1}^{N}w_i = 1  \\
     C(\textbf{w}) = \sum_{i=1}^{N}b(w_i)& \leq K, \ \  b(w_i)=
    \begin{cases} 0,&(w_i=0) \\
                  1,&(w_i>0) 
    \end{cases}
\end{aligned}
\end{equation}
where $C(\textbf{w})$ is the function for calculating cardinality of the portfolio weight $\textbf{w}$ and $K$ is integer such that $K<N$. 
\\ \\
\quad To solve the partial replication problem as a mathematical optimization problem, it is essential to express the cardinality constraint in a form that can be handled by mathematical optimization techniques, particularly in a differentiable manner. Thus we introduct the Differentiable Cardinality Constraint (\ours) in Section 4. And the \ours\ must satisfy the following properties with differentiability:
\begin{itemize}
    \item \textbf{Accuracy}: The cardinality function, $C(\textbf{w})$, in the \ours\ should convert each weight value into the corresponding integer value.
    \item \textbf{Assurance}: The \ours\ must guarantee the limit of the number of selected stocks ($K$).
\end{itemize}

\section{Differentiable Cardinality Constraint}

\begin{figure}[t]
\centering
\includegraphics[width=1\columnwidth]{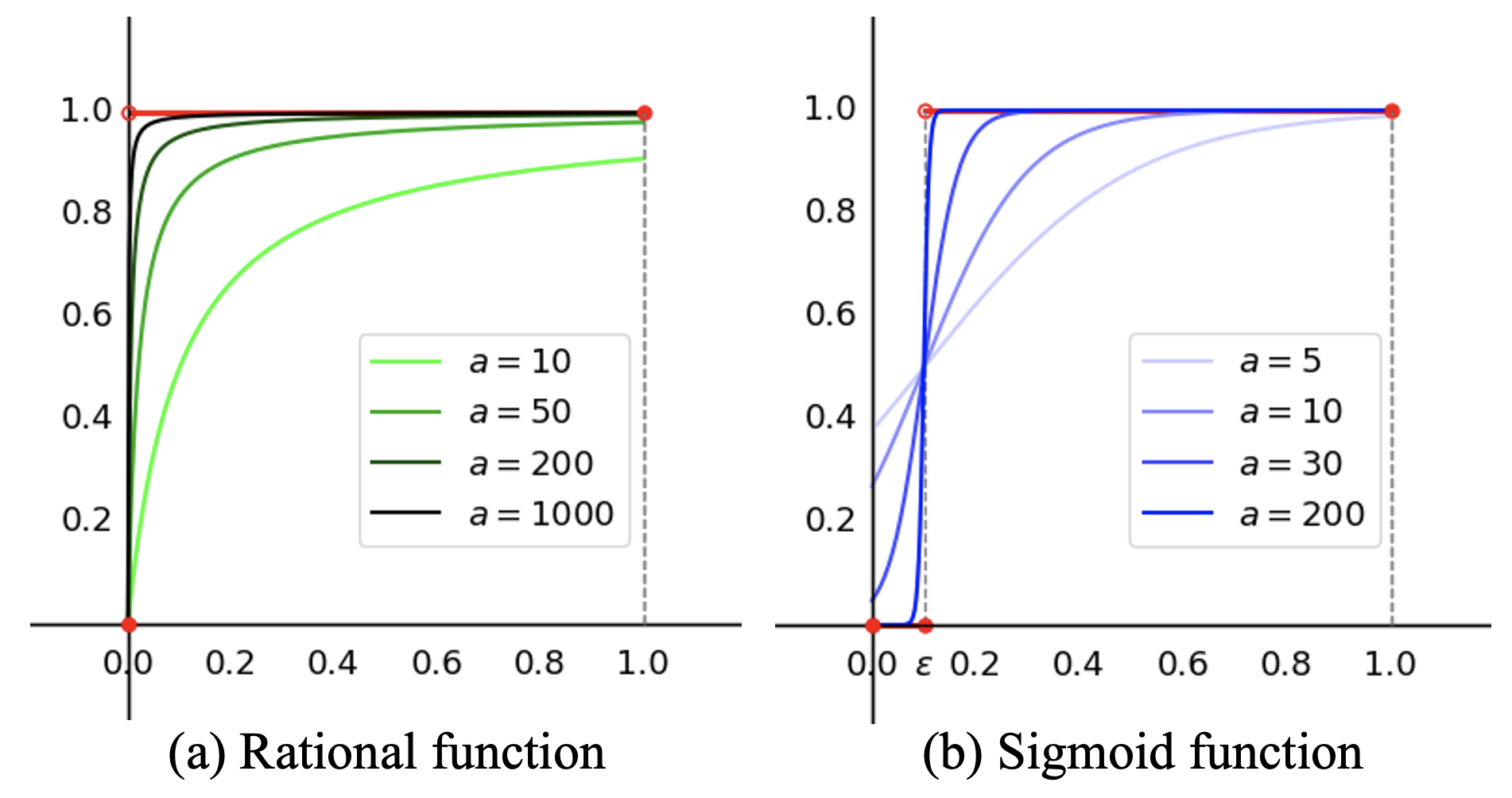}
\caption{(a) red: binary function, $b(w_i)$. green lines: rational approximated function of $b(w_i)$. (b) red: binary function with cutoff threshold, $b_{fpp}(w_i)$. blue lines: sigmoid approximated function of $b_{fpp}(w_i)$.}
\label{fig_approx_bi}
\end{figure}

\subsection{Rational Function Approach}
Defining the cardinality constraint requires a function that calculates the portfolio's cardinality, $C(\textbf{w})$. This function can be expressed as the summation of a binary function $b$. However, as illustrated in Figure~\ref{fig_approx_bi} (a) (red), this binary function is discontinuous and non-differentiable, rendering $C(\textbf{w})$, and cardinality constraints are non-differentiable as well. To address this, we approximate the binary function with a differentiable alternative, allowing $C(\textbf{w})$ and the cardinality constraint to be expressed in a differentiable form under two properties in Preliminaries (Will be discussed in Section 4.2). We utilize the following rational function to approximate the binary function:
\begin{equation}
\begin{aligned}
    & \Tilde{b}(w_i) = 1 - \frac{1}{a\cdot w_i + 1}, \quad a: \text{constant}
\end{aligned}
\end{equation}
The graph of the $\Tilde{b}(w_i)$ function is shown in Figure~\ref{fig_approx_bi} (a). This rational function $\Tilde{b}(w_i)$ is differentiable for all weight in $[0,1]$, and it passes through the origin and approaches $\Tilde{b}=1$ as an asymptote, so that it takes the value of 0 when $w_i$ is 0 and approaches 1 for $w_1$ greater than 0. Here, the constant $a$ can be arbitrarily chosen, and increasing the value of $a$ allows $\Tilde{b}(w_i)$ to approximate binary function $b(w_i)$ more closely (See Figure~\ref{fig_approx_bi} (a) (green lines)). Therefore, selecting a very large value for $a$ is advantageous. Furthermore, the value of $a$ in the approximation function remains independent of the portfolio weight or the number of stocks, thus incurring no additional computational cost or execution time as $a$ increases. 

Using $\Tilde{b}(w_i)$, the function for calculating cardinality of a portfolio can be approximated. Since $\widetilde{C}(\mathbf{w})$ is composed of differentiable functions of each variable $w_i$, it is also a differentiable $N$th-order function. Therefore, we can get the Differentiable Cardinality Constraint (\ours) using $\widetilde{C}(\mathbf{w})$:
\begin{equation}
\begin{aligned}
    \widetilde{C}(\mathbf{w})
    =\sum_{i=1}^{N} \Tilde{b}(w_i) 
    = \sum_{i=1}^{N} \left (1 - \frac{1}{a*w_i + 1} \right ) \leq K
\end{aligned}
\end{equation}


However, when selecting stocks, portfolio typically sets a weight cutoff threshold. This means that instead of strictly counting weights as 0 when they are exactly zero, the binary function should count a weight as 0 if it is below the cutoff threshold, and as 1 if it is above the threshold. This adjustment accounts for floating-point precision and requires a new binary function that incorporates the cutoff threshold.

\subsection{Sigmoid Function Approach (\oursfpp)}
\subsubsection{Cardinality Constraint with Cutoff Threshold Consideration}
As discussed, due to the floating-point precision issues from cutoff threshold, we redefine the cardinality constraint considering the cutoff threshold of portfolio weights like this:
\begin{equation}
\begin{aligned}
    C_{fpp}(\mathbf{w}) \ = \ 
    \sum_{i=1}^{N}b_{fpp}(w_i) \leq K, \\
    \text{where } b_{fpp}(w_i) = 
        \begin{cases} 0, & \text{if } 0 \leq w_i < \epsilon \\
                      1, & \text{if } w_i \geq \epsilon
        \end{cases}
\end{aligned}
\end{equation}
The graph of the redefined binary function is shown in Figure~\ref{fig_approx_bi} (b) (red). Here, $\epsilon$ represents a small cutoff threshold of portfolio weights. Therefore, the redefined cardinality function means that if a weight is less than $\epsilon$, it is counted as zero, and if it is greater than $\epsilon$, it is counted as one.
The redefined binary function $b_{fpp}$ from the $C_{fpp}$ remains a non-differentiability. We approximate this again to make it differentiable. However, we can no longer approximate the binary function using a rational function as before. Instead, we transform the sigmoid function to preserve the meaning of the $b_{fpp}(w_i)$ and make it differentiable as follows:
\begin{equation}
\begin{aligned}
    \tilde{b}_{fpp}(w_i) = \frac{1}{1 + e^{-a(w_i-\epsilon)}}, \quad a: \text{constant}
\end{aligned}
\label{sigmoid function for binary}
\end{equation}
See Figure~\ref{fig_approx_bi} (b) (blue lines). As shown in the graph, a larger value of $a$ results in a closer approximation to $b_{fpp}(w_i)$. Since $a$ is a simple constant (for the same reasons as before), choosing a large $a$ does not affect the problem's complexity or execution time. This approximated binary function is differentiable and has an inflection point at $w_i=\epsilon$. Additionally, when $a$ is set sufficiently large, the function has an asymptote at $\tilde{b}_{fpp}=1$ for weights greater than $\epsilon$ and an asymptote at $\tilde{b}_{fpp}=0$ for weights less than $\epsilon$.  Although $\tilde{b}_{fpp}(w_i)$ is 0.5 when $w_i=\epsilon$, the floating-point precision issue means that weights are rarely exactly $\epsilon$. Even if they are, the cardinality constraint is still ensured. This will be discussed further in the next section. To summarize, if a weight is greater than $\epsilon$, $\tilde{b}_{fpp}(w_i)$ is close to 1; if it is less than $\epsilon$, $\tilde{b}_{fpp}(w_i)$ is close to 0.

Similarly, the approximated cardinality function can be defined using the approximated binary function. Since $\widetilde{C}_{fpp}(\mathbf{w})$ is an $N$-th degree function composed of differentiable terms with respect to each variable $w_i$, the approximated cardinality function is also differentiable. Thus Differentiable Cardinality Constraint for floating-point precision (\oursfpp) can be written as follows:
\begin{equation}
\begin{aligned}
    \widetilde{C}_{fpp}(\mathbf{w})
    = \sum_{i=1}^{N} \tilde{b}_{fpp}(w_i)=\sum_{i=1}^{N} \frac{1}{1 + e^{-a(w_i-\epsilon)}} \leq K 
\end{aligned}
\end{equation}

\subsubsection{Conditions for Accurate Cardinality Calculation}
In the Preliminaries, one of the key properties that the \oursfpp\ must satisfy is the accurate calculation of the portfolio’s cardinality. Achieving this accurate calculation relies on the proper definition of the cardinality function, which, in turn, depends on the binary function used within it. The ability of the cardinality function to accurately reflect the true cardinality is heavily influenced by the value of the constant $a$ used in defining the binary function.
Therefore, we establish conditions for the constant $a$ that ensure the cardinality function correctly computes the portfolio’s cardinality. 

If the value of $a$ is too small, it may count values much larger than zero even when the weight is zero, or conversely, it may fail to count exactly 1 when the weight is 1. In the first case, this could lead to a situation where cardinality is calculated for all weights, regardless of whether they actually contribute to the portfolio. Therefore, to ensure accurate cardinality calculation, the value of $a$ must at least be set such that it counts 0 when the weight is zero and confidently counts 1 when the weight is 1. However, since the $\tilde{b}_{fpp}$ does not exactly take the values of 0 and 1 but instead approaches $b=0$ and $b=1$ asymptotically, we consider a bounded condition using the same cutoff threshold value $\epsilon$ as mentioned in Section 4.2. We establish the following minimum conditions:
\begin{itemize}
    \item $C_0$ : $ w_i=0 \ \forall{i} \in \{1,\cdots, N\}\Rightarrow \sum^N_{i=1}{w_i}\leq \epsilon$
    \item $C_1$ : $w_i=1 \ \forall{i} \in \{1,\cdots, N\}  \Rightarrow \sum^N_{i=1}{w_i}\geq N-\epsilon$
\end{itemize}

\subsubsection{Conditions for Assurance of the \oursfpp}

\newtheorem{theorem}{Theorem}
\newtheorem{lemma}{Lemma}

Similarly, we must ensure the second property of \oursfpp, Assurance. This can be confirmed by verifying that satisfying \oursfpp\ always guarantees the actual cardinality constraints. 
To prove that our \oursfpp\ ensures the actual cardinality constraint, we need to show the following:
\begin{align*}
    \text{If} \ \sum_{i=1}^N \tilde{b}_{fpp}(w_i) \leq K, \quad  
    \text{then} \ \sum_{i=1}^N b_{fpp}(w_i) \leq K
\end{align*}
Then we present the following theorem:
\begin{theorem} The Assurance of the cardinality constraint
    \begin{align*}
        \text{If }\ N \cdot err < 1 \quad \text{and} \quad \sum_{i=1}^N &\tilde{b}_{fpp}(w_i) \leq K,\\
        &\text{then }\ \sum_{i=1}^N b_{fpp}(w_i) \leq K 
    \end{align*}
$N$ is the number of stocks, $e$ is constant in Lemma 1 (You can see in Appendix), and $K$ is integer such that $K<N$.
\end{theorem}
\begin{proof}
The proof of this theorem, along with the necessary lemmas, and related details, can be found in the appendix.
    \label{Theorem_main}
\end{proof}
Since $N$ is fixed value and $err$ is depend on the constant $a$ in Eq.~\ref{sigmoid function for binary}, if we choose $a$ such that $err<\frac{1}{N}$, then our \oursfpp \ will always ensure the cardinality constraint.
Through Theorem~\ref{Theorem_main}, we have proven that our \oursfpp\ guarantees the cardinality constraint. In other words, our proposed \oursfpp\ can effectively solve cardinality problem by applying to some optimization algorithms without calculating $l_0$-norm.
This theorem leads to the derivation of last condition for the hyperparameter $a$ that ensure the accurate enforcement of the cardinality constraint.
\begin{itemize}
    \item $C_2$ : $err\leq \frac{1}{N}$
\end{itemize}

\subsubsection{Complex Analysis}
We applied our proposed \oursfpp\ to Sequential Least Squares Quadratic Programming (SLSQP), a mathematical optimization technique using the Lagrangian multiplier method. The time complexity analysis indicates that the addition of \oursfpp\ to SLSQP maintains the algorithm's polynomial time complexity.
For a detailed complexity analysis and the corresponding algorithm, please refer to the Appendix (see Algorithm~\ref{alg:algorithm}).

\section{Experiments}
\quad In this section, we validate the proposed \oursfpp\ with various dataset in three aspects: 1) We compare index tracking errors to assess the performance of partial replication. 2) We measure the performance of the generated portfolio using commonly used metrics, 3) we compare the runtime of methods to highlight their efficiency. 
\subsection{Experimental Settings}
\subsubsection{Data}
We conduct experiments using the following three market indices:

\begin{enumerate}
    \item \textbf{S\&P 100 Index} 
    : The S\&P 100, i.e., Standard \& Poor's 100, is a highly representative index that focuses on large, blue-chip companies with high liquidity and stability. As the most critical dataset for our experiments, it is often regarded as a more concentrated and definitive representation of the U.S. market's most significant and stable companies.
    \item \textbf{S\&P 500 Index}
    : The S\&P 500 is a broader stock market index that tracks the performance of 500 large companies listed on U.S. stock exchanges. While it provides a comprehensive overview of the U.S. economy, the S\&P 100 is considered a more focused and essential subset within this broader index.
    \item \textbf{KOSPI 100 Index}
    : The KOSPI 100, by the Korea Exchange (KRX), tracks the top 100 large-cap stocks in the Korean market, offering insights into large-cap stock trends in South Korea.
    
\end{enumerate}
For our analysis, we utilized data spanning from January 1, 2018, to April 30, 2023. The data was sourced from Yahoo Finance~\cite{yfinance}.

\subsubsection{Backtesting}
We apply the backtesting method presented in \cite{SNN} to evaluate and compare the tracking performance of each baseline model. The backtesting is conducted using a sliding window technique, where the data period is shifted by one day at a time. On rebalancing days, which are specific days set at regular intervals, the portfolio is adjusted by recalculating and applying new asset weights based on the most recent data. 
On other days, the performance is assessed using the weights fitted on the most recent rebalancing day. In our study, we rebalance the portfolio on a quarterly basis. For each rebalancing, we utilze one year of historical data to obtain the portfolio weights.

\subsubsection{Baselines}
These baselines are chosen to show that our approach can achieve performance comparable to the state-of-the-art (SOTA) methods.

\begin{enumerate}
\item \textbf{Stochastic Neural Network-based Model (SNN)}
: The SNN model is a state-of-the-art for solving partial replication, known for returning a portfolio with high tracking performance in a short time. During model training, reparametrization is used to express parameters with constraints as unconstrained parameters.  It performs well but lacks interpretability due to being a black-box model.

\item \textbf{Forward Selection}
: The Forward Selection approach satisfies the cardinality constraint by performing $K+1$ full replications. Each iteration, the model select the highest weight from the remaining stocks with Sequential Leasts Squares Quadratic Programming (SLSQP) algorithm. This process is repeated $K$ times to select a total of $K$ stocks. SLSQP is performed on the selected $K$ stocks to obtain a portfolio that meets the cardinality constraint.

\item \textbf{ Backward Selection}
: In contrast to forward selection, the Backward selection approach performs full replication and iteratively excludes the stock with the smallest weight. This process is repeated until $K$ stocks remain. After $N-K+1$ iterations, a portfolio that satisfies the cardinality constraint is obtained.
\end{enumerate}
These baseline models are evaluated to demonstrate the efficacy and efficiency of our proposed method in achieving competitive tracking performance while meeting the necessary constraints.

\subsubsection{Ours}
For the fair comparison, we utilize SLSQP with \oursfpp \ instead of \ours \ to resolve precision issues, ensuring accurate cardinality enforcement and efficient tracking.

\subsection{Index Tracking with Cardinality Constraint}

\begin{table}[!t]
\centering
\begin{tblr}{
  column{even} = {c},
  column{3} = {c},
  hlines,
  vline{2-4} = {-}{},
  hline{1,6} = {-}{0.08em},
  hline{2} = {-}{solid,black},
  hline{2} = {2}{-}{solid,black},
}
                  & $\textbf{K=20}$ & $\textbf{K=25}$ & $\textbf{K=30}$ \\
\textbf{forward}  & 8.9069          & 8.2575          & 7.9333          \\
\textbf{backward} & 7.7373          & 7.5871          & 8.6562          \\
\textbf{SNN}      & 5.8007          & 4.0457          & 4.9006          \\
\textbf{\oursfpp}     & \textbf{3.9155} & \textbf{3.5385} & \textbf{2.3922} 
\end{tblr}
\caption{Mean Absolute Error (MAE) between the tracking index and the target index}
\label{table_MAE}
\end{table}
\begin{figure*}[!t]
\centering
\includegraphics[width=0.9 \textwidth]{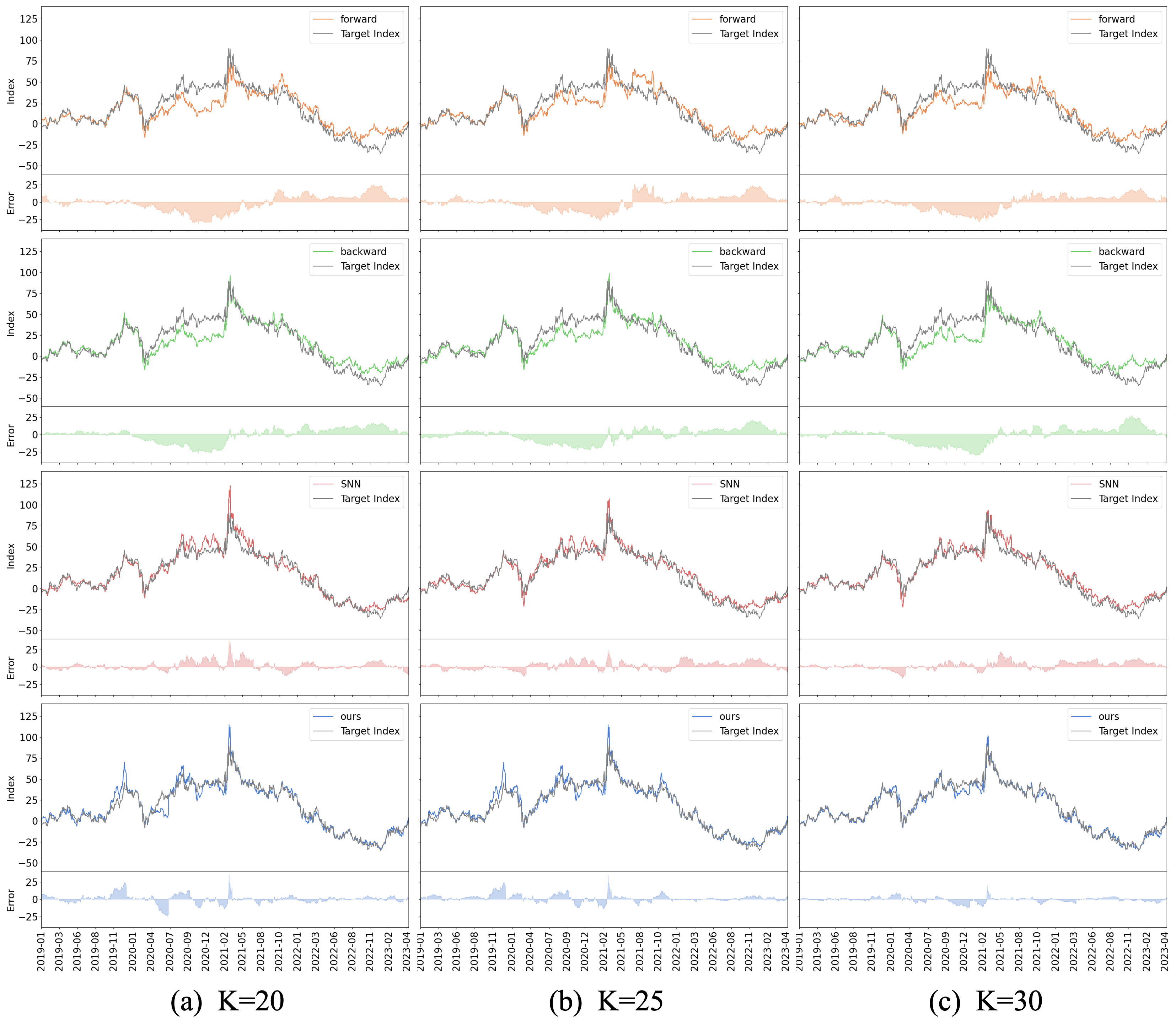}
\caption{Comparison of the tracking indices (orange, green, red, blue) with the S\&P 100 target index (grey) over time. Each row indicates SLSQP with forward selection, SLSQP with backward selection, SNN, and SLSQP with \oursfpp, respectively. Additionally, the graph below shows the absolute error between the tracking index and the target index (full cardainlity).}
\label{fig_sp100}
\end{figure*}

We measure the error between target index and tracking index of partial replication.
To evaluate the performance of index tracking. We first fit the portfolio weights on each rebalancing day (3-month) and calculate the tracking index by taking the weighted sum of the returns of each stock. We then plot this tracking index alongside the target index values to visually assess how well each method tracks the S\&P 100 index. As illustrated in Figure~\ref{fig_sp100} and summarized in Table~\ref{table_MAE}, our \oursfpp\  outperforms the baselines by accurately adhering to the cardinality constraint through a rigorous mathematical procedure.

Our SLSQP with \oursfpp\ outperforms the baselines in tracking performance by precisely adhering to the cardinality constraint. Forward and backward selection methods perform poorly because they separate portfolio fitting from asset selection, often leading to suboptimal solutions regardless of the value of $K$. In contrast, SNN and our method integrate selection and fitting simultaneously, resulting in superior tracking performance. As $K$ decreases, our method effectively reduces the number of selected assets, maintaining strong performance while naturally achieving a slight increase in error, which is expected in scenarios with fewer assets.

\begin{figure}[!t]
\centering
\includegraphics[width=1\columnwidth]{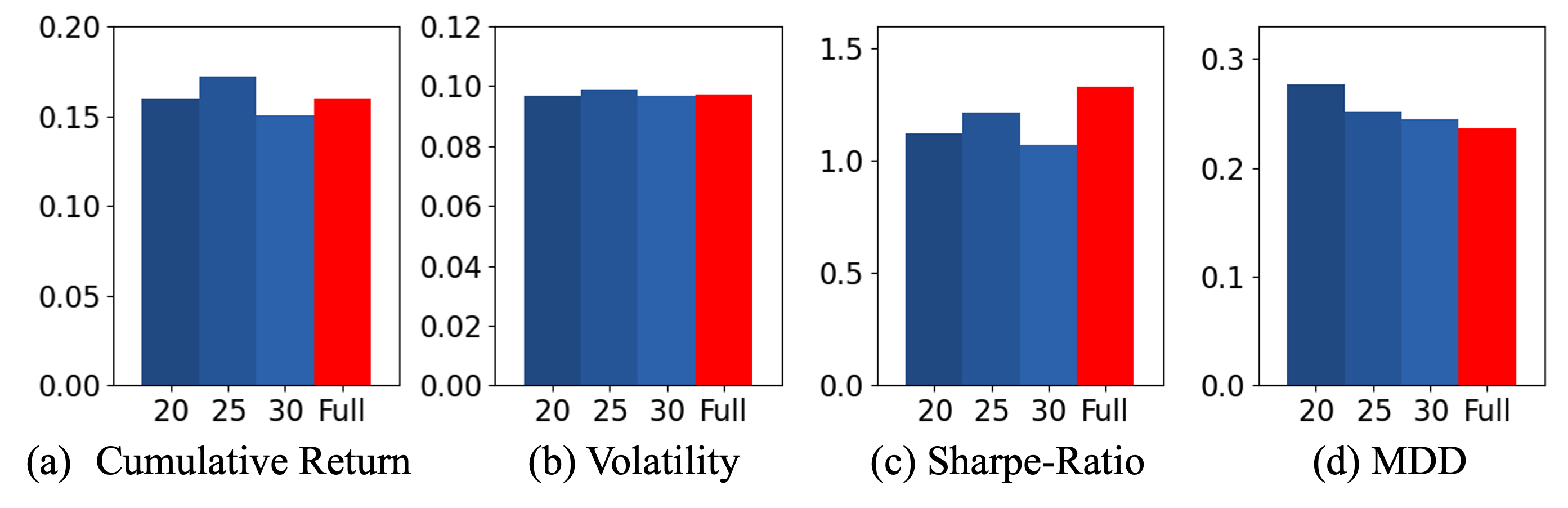}
\caption{Comparison of secondary evaluations between portfolio of \oursfpp\ at different cardinality $K=20, 25$ and $30$ and the portfolio of full replication. }
\label{fig_secondary}
\end{figure}
To assess the effectiveness of the portfolios generated through partial replication using our method, we evaluate them using commonly used metrics: cumulative return, volatility, Sharpe ratio, and maximum drawdown (MDD).
These four evaluation metrics are aggregated as the averages of the values obtained from all portfolios during the backtesting period. Our method demonstrates performance comparable to that of the full replication (See Figure~\ref{fig_secondary}). Despite the cardinality constraints, our portfolio consistently maintains a sharpe-ratio above 1, indicating that it provides favorable returns relative to its risk. Actually, the cumulative return is comparable to full replication, and the volatility shows minimal difference.

\begin{figure}[!t]
\centering
\includegraphics[width=0.99\columnwidth]{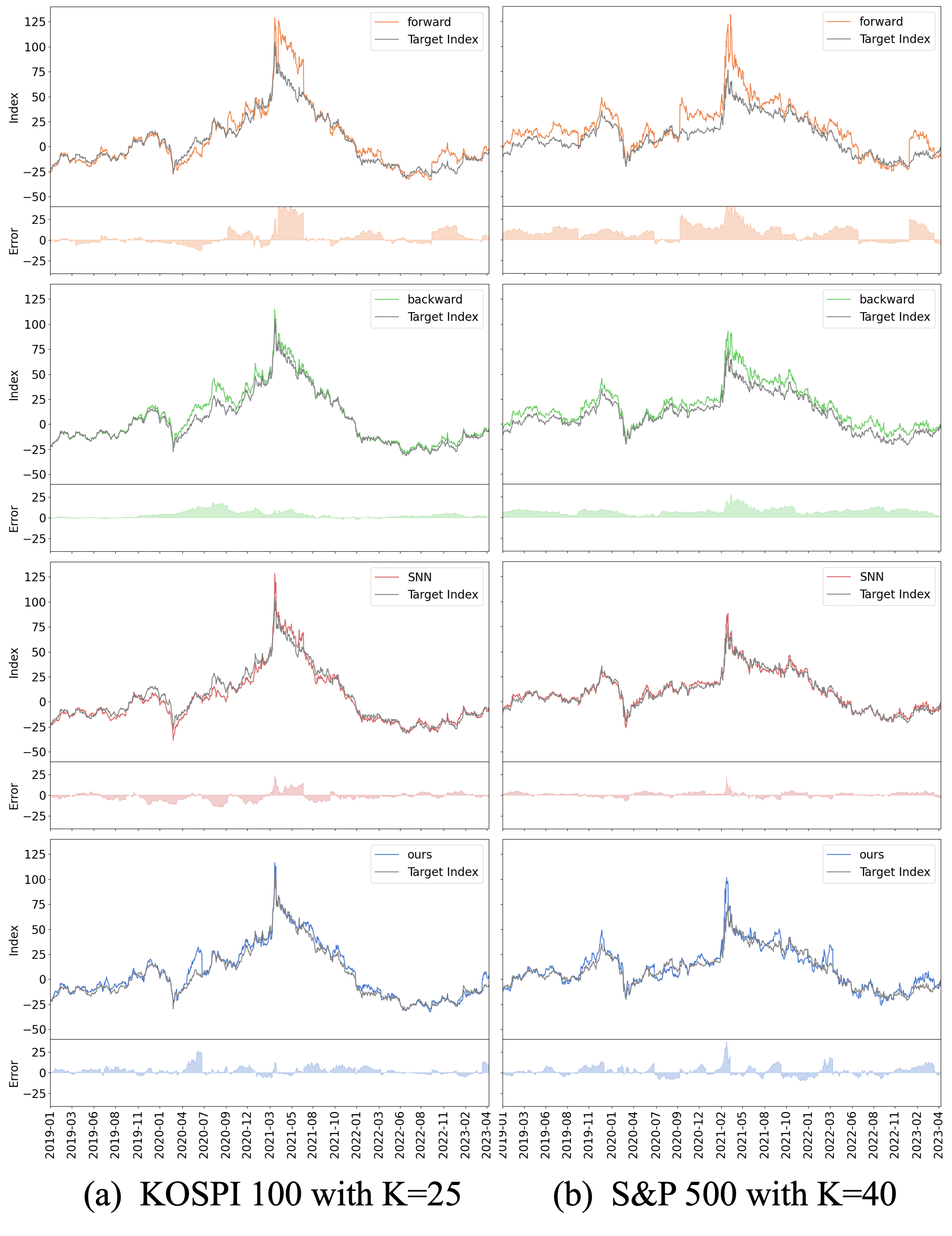}
\caption{Comparison of the tracking indices (orange, green, red, blue) on KOSPI 100 and S\&P 500. Each rows represents forward, backward, SNN, and \oursfpp, respectively.}
\label{fig_kospi100_sp500}
\end{figure}

To evaluate the robustness of the methods, we also conducted experiments about the KOSPI 100 and S\&P 500 indices. The cardinality constraint was set to $K=25$ and $K=40$, respectively, and the tracking results are shown in Figure~\ref{fig_kospi100_sp500}. Our SLSQP with \oursfpp\ effectively tracks the KOSPI 100 index despite its different distribution compared to the S\&P 100, and it also demonstrates highly comparable performance on the S\&P 500, which has five times the number of stocks as the S\&P 100. 
Our method can efficiently handle index tracking with cardinality constraints across any dataset.

\subsection{Efficiency}

\begin{figure}[!t]
\centering
\includegraphics[width=0.9\columnwidth]{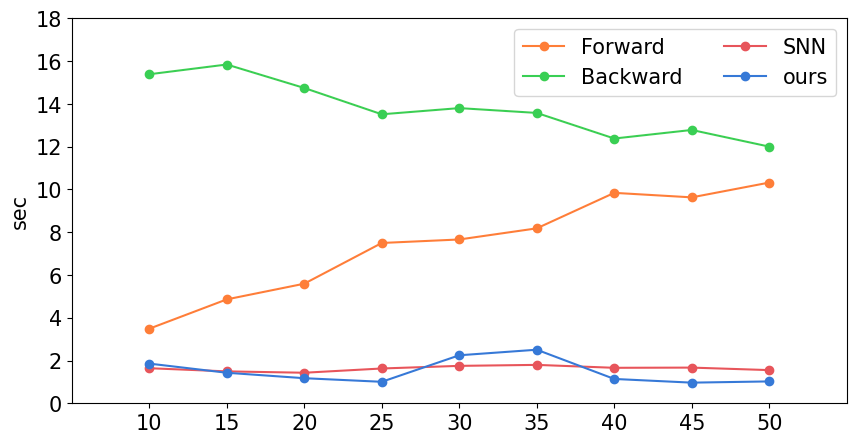}
\caption{Time comparison of each baseline. Each color indicates method: forward selection (orange), backward selection (green), SNN (red), and \oursfpp\ (blue), respectively. The $x$-axis and $y$-axis represent the cardinality $K$ and the average time (sec) taken to optimize. }
\label{fig_time_comparison}
\end{figure}

 To illustrate the efficiency of our approach, we compare the runtime taken for index tracking with a cardinality constraint using forward selection and backward selection. The original cardinality constraint is known as an NP-hard problem. Incorporating a cardinality constraint into mathematical optimization algorithms typically results in exponential complexity for the index tracking solution. By applying our proposed \oursfpp, we can find an exact solution that satisfies the existing cardinality constraint within polynomial time. 

In Figure~\ref{fig_time_comparison}, the $y$-axis represents the average runtime to optimize the weights on rebalancing days. As shown in Figure~\ref{fig_time_comparison}, our method (red) is significantly faster than the forward and backward selection. Specifically, forward selection requires $K+1$ iterations of full replication to select the portfolio weights, while backward selection requires $N-K+1$ iterations. Consequently, as $K$ increases, the runtime for forward selection increases, and the runtime for backward selection decreases. Conversely, as $K$ decreases, the runtime for forward selection decreases, and the runtime for backward selection increases. In contrast, our method maintains a consistently low runtime regardless of whether $K$ increases or decreases. 
This is because our approach directly incorporates the \oursfpp\ as a constraint in the mathematical optimization process, eliminating the need for repetitive full replications while still satisfying the cardinality constraint. 
Also, our method achieves comparable runtime efficiency to the state-of-the-art SNN model, demonstrating its superior efficiency in solving partial replication problems.

\subsection{Hyperparameter Analysis}

\begin{figure}[!t]
\centering
\includegraphics[width=0.9\columnwidth]{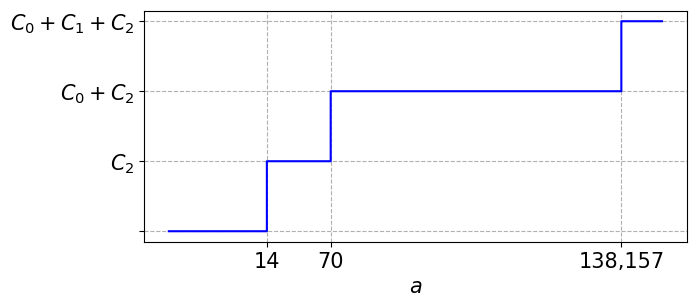}
\caption{Exploring the range of $a$ ($x$-axis) satisfying three Conditions, i.e. $C_0$, $C_1$ and $C_2$.}
\label{fig_ablation}
\end{figure}

We also explored the hyperparameter $a$ which determines the satisfiability of the two essential properties (accuracy and assurance) in \oursfpp.
To identify the appropriate $a$ value satisfying all three conditions ($C_0$, $C_1$ and $C_2$), we analyzed the status of each condition based on the value of $a$.
As shown in Figure~\ref{fig_ablation}, when $a\leq14$, none of the conditions are met. For  $a\geq 14$, $C_0$ is satisfied, $C_2$ is satisfied for $a\geq70$, and finally, $C_1$ is met when $a \geq 138,157$. $a$ should be set to at least 138,157 to satisfy all conditions in Python's 64-bit floating-point precision, ensuring that our \oursfpp\ an accurately calculates the portfolio's cardinality and guarantees the enforcement of the cardinality constraint, regardless of the dataset.

\section{Conclusion}

In this work, we introduced the Differentiable Cardinality Constraint (\ours) and its precision-aware variant (\oursfpp) to address the NP-hard problem of partial replication in index tracking. Our method converts the problem into a differentiable form, enabling efficient solutions using mathematical optimization within polynomial time. Experiments show that \oursfpp\ achieves comparable performance to full replication while adhering to cardinality constraints, outperforming state-of-the-art heuristic methods in both accuracy and efficiency. The robustness and reduced computational complexity of our approach make it highly applicable in real-world portfolio optimization.

\appendix
\section{Differentiable Cardinality Constraint for floating-point precision \oursfpp}

\subsection{Assurance of the \oursfpp}
\quad To demonstrate that our proposed Differential cardinality Constraint for floating-point precision (\oursfpp) can ensure the actual cardinality constraint, we present some theorems and prove a theorem along with the conditions on the constant $a$ used in the approximated binary function.
Before proving this, we define a few necessary Lemmas.

\begin{lemma} (Binary function error boundedness) \\
The error in the approximated binary function value for each weight is bounded by a constant value. 
    \begin{equation*}
    \begin{aligned}
        \forall i &\in \{1,\cdots,N\},\\
        &|b_{fpp}(w_i) - \Tilde{b}_{fpp}(w_i)| \leq \int_{0}^{1} |b_{fpp}(w_i) - \Tilde{b}_{fpp}(w_i)| \, dw_i 
    \end{aligned}
    \end{equation*}
where $w_i \in [0,1]$. 
\end{lemma}

\begin{proof}
Trivial. Each weight error is smaller than the sum of possible errors. The sum of possible errors is a constant, because the definite integral of function (the difference between $b_{fpp}(w_i)$ and $\Tilde{b}_{fpp}(w_i)$) is the same regardless of $i$. 
\end{proof}

\noindent We defined $\int_{0}^{1} |b_{fpp}(w_i) - \Tilde{b}_{fpp}(w_i)| \, dw_i$ as $e$.


\begin{lemma} (Cardinality function error boundedness) \\
The error in the approximated cardinality function value for some weight vectors is bounded by a constant value.
    \begin{equation*}
        \begin{aligned}
            |\sum_{i=1}^N b_{fpp}(w_i) - \sum_{i=1}^N \Tilde{b}_{fpp}(w_i)| \leq N \cdot e
        \end{aligned}
    \end{equation*}
$N$ is the number of stocks, $e$ is the constant in Lemma 1.
\end{lemma}

\begin{proof} 
By several obvious mathematical properties and Lemma 1,
    \begin{equation*}
        \begin{aligned}
            &|\sum_{i=1}^N b_{fpp}(w_i) - \sum_{i=1}^N \Tilde{b}_{fpp}(w_i)|
            \\= &|\sum_{i=1}^N (b_{fpp}(w_i) - \Tilde{b}_{fpp}(w_i))| \\
            \text{(by Lemma 1)} \quad
            \leq &\sum_{i=1}^N \int_0^1 | (b_{fpp}(w_i) - \Tilde{b}_{fpp}(w_i))|\, dw_i \\
            = &N \cdot \int_0^1 | (b_{fpp}(w_i) - \Tilde{b}_{fpp}(w_i))|\, dw_i \\
            = &N \cdot e
        \end{aligned}
        \end{equation*}
\end{proof}

\begin{lemma} (Discrete characteristics of cardinality) \\
These two conditions are equivalent. \\
    \begin{equation*}
        \begin{aligned}
            & 1) \quad \sum_{i=1}^N b_{fpp}(w_i) \leq K \\
            & 2) \quad \sum_{i=1}^N b_{fpp}(w_i) < K+1
        \end{aligned}
    \end{equation*}
\\ K is an integer such that $K<N$.
\end{lemma}
\begin{proof}
    Trivial. Since $\sum_{i=1}^N b_{fpp}(w_i)$ and $K$ are integers.
\end{proof}

To prove that our \oursfpp\ ensures the actual cardinality constraint, we need to show the following:
\begin{align*}
    \text{If} \ \sum_{i=1}^N \tilde{b}_{fpp}(w_i) \leq K, \quad  
    \text{then} \ \sum_{i=1}^N b_{fpp}(w_i) \leq K
\end{align*}
Using Lemma 2 and Lemma 3, we can obtain the following theorem:
\begin{theorem} The Assurance of the cardinality constraint
    \begin{align*}
        \text{If }\ N \cdot e < 1 \quad \text{and} \quad \sum_{i=1}^N &\tilde{b}_{fpp}(w_i) \leq K,\\
        &\text{then }\ \sum_{i=1}^N b_{fpp}(w_i) \leq K 
    \end{align*}
$N$ is the number of stocks, $e$ is constant in Lemma 1, and $K$ is integer such that $K<N$.
\end{theorem}
\begin{proof}
    \begin{equation*}
    \begin{aligned}
        &\text{Suppose that } N \cdot e < 1 \ \text{ and } \ \sum_{i=1}^N \tilde{b}_{fpp}(w_i) \leq K. \\
        &\text{By Lemma 2, } \ |\sum_{i=1}^N b_{fpp}(w_i) - \sum_{i=1}^N \Tilde{b}_{fpp}(w_i)| \leq N \cdot e. \\
        &\Rightarrow \sum_{i=1}^N \Tilde{b}_{fpp}(w_i) - N \cdot e \leq \sum_{i=1}^N b_{fpp}(w_i) \\
        &\leq \sum_{i=1}^N \Tilde{b}_{fpp}(w_i) + N \cdot e \\
        &\Rightarrow  \sum_{i=1}^N b_{fpp}(w_i) \leq \sum_{i=1}^N \Tilde{b}_{fpp}(w_i) + N \cdot e \leq K + N \cdot e. \  \\ \\
        &\text{By assumption,} \quad K + N \cdot e < K + 1. \\
        &\text{Thus,} \quad \sum_{i=1}^N b_{fpp}(w_i) < K + 1. \\
        &\text{By Lemma 3,} \quad \sum_{i=1}^N b_{fpp}(w_i) \leq K.
    \end{aligned}
    \end{equation*}
    \label{Theorem}
\end{proof}
Since $N$ is fixed value and $e$ is dependent on the constant $a$ in Eq.~\ref{sigmoid function for binary}, if we choose $a$ such that $e<\frac{1}{N}$, then our \oursfpp \ will always ensure the cardinality constraint.
Through Theorem~\ref{Theorem}, we have proven that our \oursfpp\ guarantees the cardinality constraint. In other words, our proposed \oursfpp\ can effectively solve cardinality problem by applying to some optimization algorithms without calculating $l_0$-norm.

\subsection{Complex Analysis}
When applying our proposed \oursfpp\ to the mathematical optimization method, we analyze the time complexity of partial replication using Lagrange multipliers to show that the complexity of the optimization algorithm remains polynomial time. To analyze the time complexity from a computational implementation perspective, we applied the \oursfpp\ to Sequential Least Squares Quadratic Programming (SLSQP), a mathematical optimization technique using the Lagrangian multiplier method. 
First, the algorithm for solving partial replication using the SLSQP, which is detailed in Algorithm~\ref{alg:algorithm} (see Appendix), serves as the foundation for the subsequent complexity analysis.

\begin{enumerate}
    \item Initializing a variable of order $N$ is $O(N)$. 
    \item Defining the Lagrangian function is just for understanding next step. Thus, actually, this step is not executed.
    \item To calculate the $\frac{\partial L}{\partial \textbf{w}}$, we have to calculate $R^TR$ and $R^T\mathbf{y}$. These are $O(DN^2)$ and $O(DN)$, respectively. 
    \item Setting the Karush-Kuhn-Tucker (KKT) conditions and solving these are just $O(1)$.
    \item To solve the system of equations of 3 and 4, we utilize SLSQP solver. SLSQP solver commonly performs matrix operations internally, with a time complexity of $O(N^3)$.
\end{enumerate}
The combined time complexity for each step of the algorithm is as follows:
\begin{itemize}
    \item $O(N)+O(DN^2)+O(DN)+O(1)+O(N^3)=O(N^3)$
\end{itemize}
Therefore, we can conclude that adding \oursfpp\ to SLSQP allows us to solve the problem in polynomial time.

\begin{algorithm}[tb]
\caption{Partial Replication using Lagrangian Multiplier and \oursfpp}
\label{alg:algorithm}
\textbf{Input}: $R \in \mathbb{R}^{D\times N}, \mathbf{y}\in \mathbb{R}^{D}, N\in \mathbb{R}, D\in \mathbb{R}$\\
\textbf{Parameter}: $\lambda\in \mathbb{R}, \boldsymbol{\mu}\in \mathbb{R}^{N}, \nu\in \mathbb{R}$\\
\textbf{Output}: $\mathbf{w}\in [0,1]^N$
\begin{algorithmic}[1] 
\STATE \textbf{Initialize} $\mathbf{w}=[\frac{1}{N}, ..., \frac{1}{N}], \boldsymbol{\mu}=[0,...,0], \nu=0$
\STATE \textbf{Define the Lagrangian Function}: \\ $L(\mathbf{w},\lambda,\boldsymbol{\mu},\nu)=\frac{1}{2} (R\mathbf{w}-\mathbf{y})^T(R\mathbf{w}-\mathbf{y})+\lambda(1-sum(\mathbf{w})-\boldsymbol{\mu}\cdot\mathbf{w}+\nu(\frac{1}{1+e^{-a(\mathbf{w}-\alpha)}}-K)$
\STATE \textbf{Set the partial derivatives of $L$ and them to zero:} \\ 
\quad $\frac{\partial L}{\partial \textbf{w}} = R^T(R\textbf{w}-\mathbf{y})+\lambda-\boldsymbol{\mu}+\nu\frac{-a\cdot e^{-a(\textbf{w}-\alpha)}}{(1+e^{-a(\textbf{w}-\alpha)})^2}=0$\\ 
\quad $\frac{\partial L}{\partial \lambda} = 1-sum(\textbf{w})=0$,
\quad $\frac{\partial L}{\partial \boldsymbol{\mu}} = -\textbf{w}=0$ \\ 
\quad $\frac{\partial L}{\partial \nu} = \frac{1}{1+e^{-a(\textbf{w}-\alpha)}}-K=0$ \\
\STATE \textbf{Set the KKT conditions:} \\ 
\quad $\mu_i\geq0, \nu\geq0, $ $\forall i \in [1,\dots,N]$ \\
\quad $\mu_iw_i=0, $ $\forall i \in [1,\dots,N]$, \ $\nu\frac{-a\cdot e^{-a(\textbf{w}-\alpha)}}{(1+e^{-a(\textbf{w}-\alpha)})^2}=0$
\STATE \textbf{Solve the systems of equations of $3,4$ by using SLSQP solver} \\ 
\end{algorithmic}
\end{algorithm}

\subsection{Approximation Error of Partial Replication}
The approximation error of partial replication compared to full replication can be bounded as follows:\\
Let $N$ be the number of total stocks and $K$ be the integer such that $K<N$. For each $i\in [1,2,...,N]$ and $j\in[1,2,...,K]$.
Let $err_i$ be the maximum error of each stock such that $err_i=max_j|r_i - r_j|\cdot w_i$. Then, we can get the upper bound of the tracking error of partial replication, which is described as the sum of maximum errors, $\sum_i^N{err_i}$.

\section*{Acknowledgements}
This work was supported by the Institute of Information \& Communications Technology Planning \& Evaluation (IITP) grant funded by the Korea government (MSIT) (No.2022-0-00680, Abductive inference framework using omni-data for understanding complex causal relations), the National R\&D Program through the National Research Foundation of Korea (NRF) funded by the Ministry of Science and ICT (RS-2024-00407282), and the Artificial Intelligence Convergence Innovation Human Resources Development (IITP-2025-RS-2023-00255968) funded by the Korea government (MSIT).

\bibliography{aaai25}

\end{document}